\documentclass[10pt,twocolumn,letterpaper]{article}

\usepackage{cvpr}      
\definecolor{cvprblue}{rgb}{0.21,0.49,0.74}
\usepackage[pagebackref,breaklinks,colorlinks,allcolors=cvprblue]{hyperref}

\def\confName{CVPR}
\def\confYear{2026}
\usepackage{subcaption}
\usepackage{amssymb}
\usepackage{mathtools}
\usepackage{amsthm}
\usepackage{times}
\usepackage{helvet}
\usepackage{courier}
\usepackage{xcolor}
\usepackage{graphicx}
\usepackage[accsupp]{axessibility} 

\usepackage{algorithm}
\usepackage{algorithmic}
\usepackage{bbold}
\usepackage{booktabs} 
\usepackage{resizegather}
\usepackage{graphicx} 

\usepackage{graphicx} 
\usepackage{amsmath,amsfonts,bm}
\usepackage{amsthm}
\usepackage{algorithm}
\usepackage{algorithmic}

\newtheorem{proposition}{Proposition}

\title{Bi-Level Optimization for Single Domain Generalization }

\author{Marzi Heidari$^*$, Hanping Zhang$^*$, Hao Yan$^*$, Yuhong Guo$^{*\dagger}$\\
$^*$School of Computer Science, Carleton University, Ottawa, Canada  \\ 
$^\dagger$Canada CIFAR AI Chair, Amii, Canada \\
}

\begin{document}

\maketitle

\begin{abstract}
Generalizing from a single labeled source domain to unseen target domains, without access to any target data during training, remains a fundamental challenge in robust machine learning. 
We address this underexplored setting, known as Single Domain Generalization (SDG), by proposing BiSDG, a bi-level optimization framework that explicitly decouples task learning from domain modeling. BiSDG simulates distribution shifts through surrogate domains constructed via label-preserving transformations of the source data. 
To capture domain-specific context, we propose a domain prompt encoder that generates lightweight modulation signals 
to produce augmenting features via
feature-wise linear modulation.
The learning process is formulated as a bi-level optimization problem: 
the inner objective optimizes task performance under fixed prompts, while the outer objective 
	maximizes generalization across the surrogate domains
by updating the domain prompt encoder.
We further develop a practical gradient approximation scheme that enables efficient bi-level training without second-order derivatives. Extensive experiments on various SGD benchmarks demonstrate that BiSDG consistently outperforms prior methods, setting new state-of-the-art performance in the SDG setting.

\end{abstract}
\section{Introduction}

Traditional deep learning methods typically assume that training and testing data share the same distribution, resulting in limited generalization to out-of-domain scenarios. Domain Generalization (DG) addresses this limitation by training models that can generalize to unseen target domains. Recent DG approaches either aim to learn domain-invariant representations via feature alignment \citep{li2018deep, mahajan2021domain}, simulate train-test splits through meta-learning \citep{li2018learning, du2020learning}, or seek flat minima of empirical risk to improve robustness \citep{cha2021swad, zhang2023flatness}.
However, conventional DG assumes access to data from multiple source domains, which restricts its applicability in data-sparse scenarios. Additionally, the typical evaluation setting involves testing on a single target domain, limiting assessment of generalization to diverse out-of-domain scenarios.

Single Domain Generalization (SDG) addresses these limitations by training on a single source domain while evaluating on multiple unseen target domains, thus posing a more challenging and realistic setting. Since standard DG approaches often rely on multi-domain data, they are not directly applicable to SDG. Early SDG methods \citep{devries2017improved, cubuk2019autoaugment, cubuk2020randaugment} introduce various data augmentation techniques to enhance in-domain generalization and robustness to corruptions. Other methods \citep{volpi2018generalizing, Long2020Maximum, zheng2024advst} adopt adversarial data augmentation to improve 
out-of-domain generalization performance. Additionally, generative approaches \citep{qiao2020learning, li2021progressive, chen2023meta} synthesize auxiliary data to aid generalization from a single domain.
To date, most existing SDG methods adopt target-agnostic strategies, due to the constraint of not accessing target domains. 
Such strategies face a “lottery ticket” dilemma: models tend to perform well on target domains similar to the augmented source domain, but poorly on dissimilar ones.

This motivates us to propose a domain-aware method for SDG that respects the constraint of non-access to target domains,
while achieving improved self-adaptability.
We explicitly extract domain knowledge using a domain prompt encoder and inject this knowledge into the feature representations. 
The prediction model is trained to process these fused features effectively. 
Since a single domain cannot support training a generalizable domain prompt encoder alone, we simulate a diverse set of surrogate domains by applying groups of semantically coherent data transformations to the source domain data. Inspired by FiLM \citep{perez2018film}, we modulate features using domain-specific prompts via feature-wise linear modulation.

A key challenge is that joint optimization of the task model and the domain prompt encoder does not guarantee meaningful domain knowledge extraction or effective feature fusion to support generalization. 
This often leads to a trivial solution where both components collapse into a naive data-augmentation-based model. In practice, the task model benefits from a well-trained domain prompt encoder, while the encoder’s effectiveness also depends on a well-optimized task model.
To address this interdependence, we formulate SDG as a bi-level optimization problem. 
The inner objective optimizes the task model using the original data from the training domain, while the outer objective optimizes the domain prompt encoder using data from the synthetic surrogate domains. We solve this bi-level objective using gradient estimation via 
implicit differentiation and 
gradient approximation.
We evaluate our proposed method, Bi-level Optimization for Single Domain Generalization (BiSDG), under the standard SDG setting and compare it with existing baselines. Experimental results demonstrate that BiSDG achieves state-of-the-art performance on widely used SDG benchmarks.

\section{Related Works}
\subsection{Single-Domain Generalization}
Single Domain Generalization (SDG) aims to train models that can generalize to unseen test domains using data from only a single training domain. 
This setting is more challenging than traditional multi-source domain generalization, which assumes access to data from multiple training domains.
Existing SDG approaches generally fall into three main groups.

The first group focuses on traditional data augmentation techniques to enhance in-domain generalization and robustness to corruptions.
For instance, CutOut \citep{devries2017improved} improves regularization by randomly masking square regions in input images, encouraging robust feature learning. MixUp \citep{zhang2018mixup} improves generalization by training on convex combinations of input pairs and their labels. AugMix \citep{hendrycks2019augmix} enhances robustness and uncertainty estimation by stochastically mixing multiple augmentations of an image and applying a consistency loss. AutoAugment \citep{cubuk2019autoaugment} learns optimal augmentation policies by exploring combinations of transformations with different probabilities and magnitudes, while RandAugment \citep{cubuk2020randaugment} simplifies the process by removing the search phase and using a reduced, interpretable space of augmentations.
Although these techniques improve robustness and in-domain performance, they do not explicitly address out-of-domain generalization.
A more recent approach, ACVC \citep{cugu2022attention}, applies traditional visual corruptions and enforces attention consistency between original and corrupted versions to improve domain generalization.

A second line of research introduces adversarial data augmentation methods aimed at out-of-domain generalization.
ADA \citep{volpi2018generalizing} proposes an iterative strategy that generates hard adversarial examples from fictitious target domains, using only single-source training data.
ME-ADA \citep{Long2020Maximum} incorporates a regularization term derived from the information bottleneck principle to encourage high-entropy perturbations, improving robustness to domain shifts.
\citet{zhang2023adversarial} propose adversarial perturbations applied to feature statistics, enabling models to learn representations that are robust to style variations. 
AdvST \citep{zheng2024advst} treats standard augmentations as learnable semantic transformations and uses adversarial training to create diverse semantic variants, optimizing a distributionally robust objective to enhance generalization to unseen domains.

A third group of methods leverages generative modeling to synthesize auxiliary training data that simulate domain shifts.
M-ADA \citep{qiao2020learning} generates challenging fictitious domains through adversarial training in a meta-learning framework, guided by a Wasserstein Auto-Encoder. 
L2D \citep{wang2021learning} introduces a style-complement module that diversifies training data by synthesizing style-varied but semantically consistent images using mutual information. 
PDEN \citep{li2021progressive} 
expands the training domain through photometric and geometric transformations, guided by contrastive learning to enforce class separation and improve generalization. 
MCL \citep{chen2023meta} applies meta-causal learning to infer and align causal factors underlying domain shifts by simulating auxiliary domains. 
\citet{xu2023simde} generate low-confidence samples by maximizing entropy and minimizing cross-entropy, using dual-view generators guided by a classifier.

Despite their differences, most existing SDG methods adopt a target-agnostic design philosophy due to the constraint of non-access to target domains. This results in what we call the “lottery dilemma”: models tend to perform better on target domains that are similar to the augmented training data, and worse on dissimilar ones.
To address this limitation, we propose a domain-aware approach to SDG that explicitly extracts domain knowledge and injects it into the learning process
to enhance adaptation and generalization. 
Our method uses a bi-level optimization framework,
where a domain prompt encoder captures domain-specific knowledge,
and the task model learns to effectively process the fused representation. 
This approach allows us to simulate domain shifts and enforce knowledge transfer, while respecting the constraint of not accessing target domain data. 

\subsection{Bi-Level Optimization}
Bi-level optimization is a powerful framework that enables the joint optimization of a nested pair of objectives, where the outer (upper-level) objective depends implicitly on the solution of the inner (lower-level) problem. This approach has been widely adopted across a range of tasks, such as hyperparameter tuning \cite{pedregosa2016hyperparameter}, neural architecture search \cite{liu2019darts}, and semi-supervised learning~\cite{heidari25}, due to its ability to model hierarchical dependencies in learning systems.
Recent works leverage bi-level optimization to tackle domain generalization across practical scenarios. 
\citet{jia2024meta} introduce a meta-learning algorithm with inner-loop and outer-loop objectives on a discrepancy measure to learn invariant representations. By adversarially minimizing both covariate and conditional shifts between source and simulated target domains in a bi-level setup, their method improves out-of-domain feature alignment and model robustness. To address data scarcity, \citet{qin2023bi} propose a two-tier meta-learning approach for few-shot domain generalization: an inner loop learns domain-specific embedding subspaces while an outer loop learns a shared base space, jointly optimized in a bi-level manner.
Beyond closed-set shifts, \citet{peng2024advancing} tackle open-set domain generalization with an evidential bi-level domain scheduler. Collectively, these studies demonstrate that integrating bi-level optimization into domain generalization pipelines can address a spectrum of practical challenges, leading to improved performance on unseen domains.

\section{Method}
\begin{figure*}[t]
\centering
\includegraphics[width=0.95\linewidth]{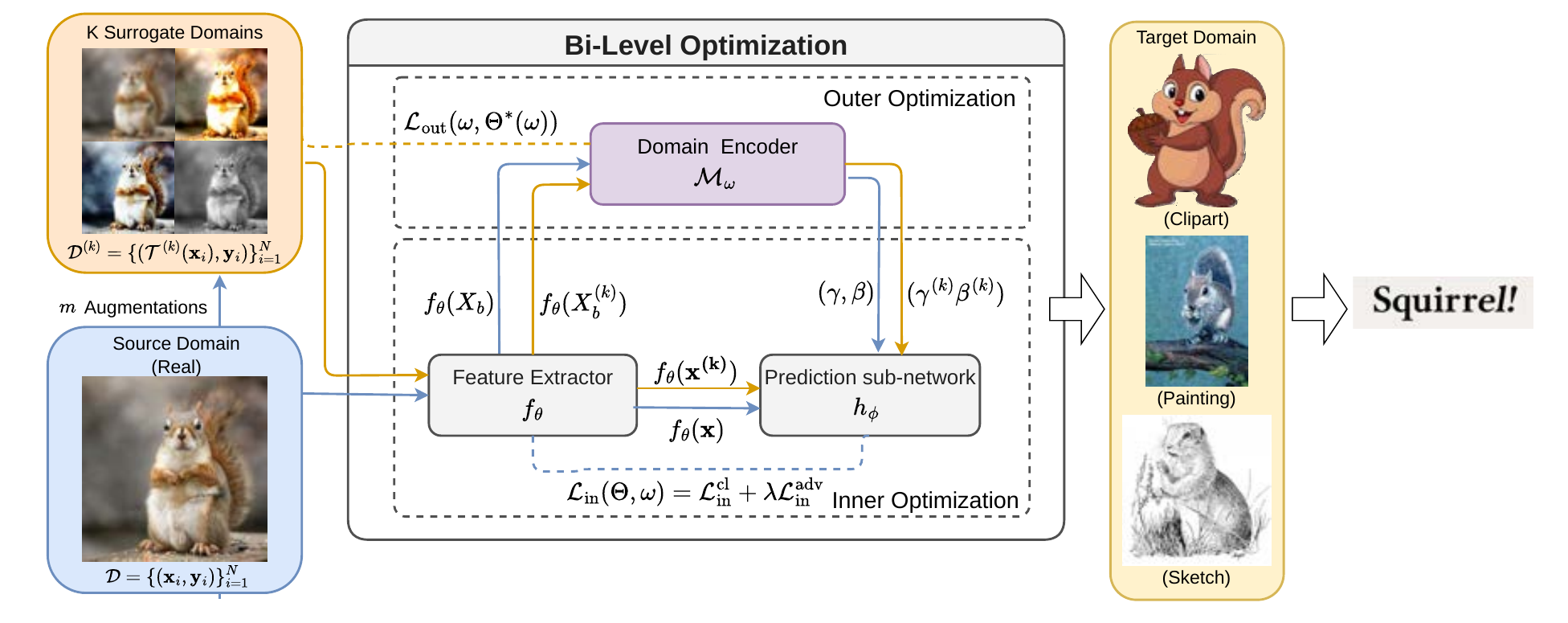}
\caption{
Overview of the BiSDG framework for Single Domain Generalization (SDG). Starting from a labeled source dataset, 
BiSDG constructs multiple surrogate domains via semantically coherent, label-preserving transformation pipelines to simulate unseen distribution shifts. 
A domain prompt encoder is introduced to capture domain-specific knowledge, which is deployed to 
perform FiLM-based feature modulation, supporting generalizable prediction. 
BiSDG formulates SGD as a bi-level optimization problem: 
	the inner objective updates the model parameters $\Theta=(\theta, \phi)$ 
	using the source domain, combining supervised cross-entropy loss $\mathcal{L}_{\text{in}}^{\text{cl}}$ and adversarial consistency loss $\mathcal{L}_{\text{in}}^{\text{adv}}$. 
	The outer objective updates the prompt encoder parameters $\omega$ to maximize generalization across surrogate domains. 
}
\label{fig:model}
\end{figure*}
\subsection{Problem Setup}
We study the task of Single Domain Generalization (SDG), where the objective is to train a model using labeled data from a single source domain and achieve robust performance on unseen target domains. 
Formally, let the source domain dataset be denoted by $\mathcal{D} = \{(\mathbf{x}_i, \mathbf{y}_i)\}_{i=1}^N$, where each input-label pair $(\mathbf{x}_i, \mathbf{y}_i)$ is drawn i.i.d. from a distribution $\mathcal{P}_{\mathrm{s}}$ over $\mathcal{X} \times \mathcal{Y}$. Our goal is to learn a prediction function $h_{\phi} \circ f_{\theta} : \mathcal{X} \rightarrow \mathcal{Y}$ that generalizes well to samples from an unknown target distribution $\mathcal{P}_{\mathrm{t}} \neq \mathcal{P}_{\mathrm{s}}$, which remains inaccessible during training.
The prediction model is composed of a feature extractor $f_{\theta}: \mathcal{X} \rightarrow \mathcal{Z}$, parameterized by $\theta$, and a classifier sub-network $h_{\phi}: \mathcal{Z} \rightarrow \mathcal{Y}$, parameterized by $\phi$. The primary challenge of SDG is that no domain alignment, calibration, or adaptation is possible using target data; instead, the model must learn domain-invariant representations solely from the available source domain data. 

\subsection{Approach}
We propose {Bi-Level Optimization for Single-Domain Generalisation} (BiSDG), a principled framework designed to tackle the challenge of generalizing from a single labeled domain to unseen target distributions. The core idea is to simulate plausible domain shifts during training and explicitly disentangle the learning of semantic representations from the modeling of domain-specific variations.
To simulate distribution shifts, BiSDG constructs a diverse set of {surrogate domains} via stochastic, label-preserving transformations applied to source data. These synthetic domains expose the model to a controlled spectrum of low-level perturbations without requiring access to target samples.
To capture domain-specific context, BiSDG introduces a { domain prompt} module that extracts a compact latent embedding from each domain. This embedding is used to condition the feature extractor via FiLM-based modulation, enabling adaptive feature normalization tailored to each domain's characteristics.
Crucially, BiSDG formulates the learning process as a bi-level optimization problem. The lower-level objective updates the feature extractor and classifier to minimize task loss under fixed domain prompts, while the upper-level objective updates the prompt encoder to maximize generalization across the surrogate domains. This separation allows the backbone to focus on learning robust task-relevant features, while the prompt encoder is guided to emphasize variations that challenge generalization. The overall framework in shown in Figure~\ref{fig:model}.

\subsubsection{Surrogate-Domain Synthesis}
In the single-domain generalization setting, the absence of target-domain samples prohibits direct exposure to distribution shift. To mitigate this, we construct a diverse collection of {surrogate domains} that serve as plausible stand-ins for unknown target environments. These surrogate domains are derived by applying label-preserving image transformations to source data, thereby inducing systematic low-level variations while preserving semantic content.

Specifically, we define a transformation pool $\{\mathcal{A}_i\}_{i=1}^{N^{\text{aug}}}$ containing a diverse set of photometric and geometric augmentations, such as Gaussian blur, color jitter, and histogram equalization. Rather than selecting augmentations at random, we group them into semantically coherent transformation pipelines, where each surrogate domain $\mathcal{D}^{(k)}$ is associated with a specific type of shift, e.g., texture, photometric, or geometric. Each pipeline $\mathcal{T}^{(k)}$, composed of $m$ augmentations, is applied consistently to all source samples to construct the $k$-th surrogate domain. Formally, we define:
\begin{equation}
  \label{eq:surrogate}
 \!\!\!
	\mathcal{D}^{(k)}=
  \bigl\{(\mathbf{x}^{(k)}_{i},\mathbf{y}_{i}):
        \mathbf{x}^{(k)}_{i}=\mathcal{T}^{(k)}(\mathbf{x}_{i}),
        (\mathbf{x}_{i},\mathbf{y}_{i})\in\mathcal{D}\bigr\}.
\end{equation}
 By deploying $K$ such transformation pipelines, we obtain a family of surrogate domains $\{\mathcal{D}^{(k)}\}_{k=1}^{K}$ that span a broad spectrum of visual attributes. This augmentation-driven diversification compels the model to learn features that are stable under a variety of appearance shifts, thereby simulating the challenge of generalizing to unseen domains.

\subsubsection{Domain Prompt Encoder}

A major challenge in training a shared feature extractor across multiple domains lies in the risk of representation collapse, wherein the network converges to features that reflect an average over all domains, thereby suppressing meaningful domain-specific variations. To retain sensitivity to these variations while preserving parameter efficiency, we introduce a domain prompt encoder $\mathcal{M}_\omega$ that learns to generate 
domain-aware modulation signals.

For each surrogate domain $\mathcal{D}^{(k)}$, we associate a latent prompt 
instantiated through a pair of modulation vectors: a scaling vector $\boldsymbol{\gamma}^{(k)}$ and a shifting vector $\boldsymbol{\beta}^{(k)}$. These are inferred from a mini-batch $X_b^{(k)} = \{\mathbf{x}_j^{(k)}\}_{j=1}^{N^\mathcal{B}}$ sampled from $\mathcal{D}^{(k)}$. To encode the domain-level distributional context in a permutation-invariant manner, we apply a domain encoder $\mathcal{M}_{\omega}$ over frozen features extracted by a stop-gradient copy of the backbone, denoted $f_{\bar{\theta}}$:
\begin{equation}
  (\boldsymbol{\gamma}^{(k)},\boldsymbol{\beta}^{(k)})
  =
  \mathcal{M}_{\omega}\bigl(f_{\bar{\theta}}(X_{b}^{(k)})\bigr).
\end{equation}
Here, $\mathcal{M}_{\omega}$ is implemented as a Set Transformer~\cite{lee2019set}, which naturally handles unordered inputs and enables the prompt to summarize the aggregate style of the domain batch.

Given an input $\mathbf{x}_i^{(k)}$ from the surrogate domain $k$, 
we extract its features 
$\mathbf{z}_i^{(k)} = f_{\theta}(\mathbf{x}_i^{(k)})$ and apply feature-wise linear modulation (FiLM) \cite{perez2018film} using the prompt-specific parameters,
while the features are standardized using batch-wise statistics $(\boldsymbol{\mu}^{(k)}, \boldsymbol{\sigma}^{(k)})$ before modulation:
\begin{equation}
	\label{eq:modulation}
\tilde{\mathbf{z}}_{i}^{(k)}
= \boldsymbol{\gamma}^{(k)} \odot \left( \frac{\mathbf{z}_{i}^{(k)} - \boldsymbol{\mu}^{(k)}}{\boldsymbol{\sigma}^{(k)}} \right)
+ \boldsymbol{\beta}^{(k)},
\end{equation}
where $\odot$ denotes element-wise multiplication. The resulting modulated features $\tilde{\mathbf{z}}_i^{(k)}$ are then fed into a domain-aware prediction head that conditions jointly on the 
original instance features and the domain prompt
modulated features. 
We denote this modulated prediction sub-network as $h_{\phi}$
that includes both the modulation process in Eq.~\eqref{eq:modulation}
and the prediction head, 
yielding predictions of the form 
$\hat{\mathbf{y}}_i^{(k)} = h_{\phi}(f_{\theta}(\mathbf{x}_i^{(k)}), \mathcal{M}_{\omega}(f_{\bar{\theta}}(X_b^{(k)})))$.

This FiLM-based strategy equips the model with lightweight, domain-adaptive behavior while keeping the shared backbone fixed across domains. Importantly, the modulation is global (feature-wise) rather than spatial, ensuring parameter efficiency and compatibility with standard architectures. 
By conditioning the model on learned domain prompts, we enable domain-informative inference to enhance 
adaptation and generalization.

\subsubsection{Bi-Level Optimization for SDG}
\label{sec:bi-level-formulation}

We formulate \textit{Single-Domain Generalization} (SDG) as a bi-level optimization problem, in which the goal is to simultaneously learn a robust prediction model that generalizes well across domain shifts, and
a domain prompt encoder that tackles domain variability via feature modulation. 
To decouple these two objectives, we introduce a bi-level formulation where 
the prediction model parameters are optimized at the inner level for task performance, while the prompt encoder is optimized at the outer level to maximize cross-domain generalization.

In this setting, the prediction model is parameterized by $\Theta = \{\theta, \phi\}$, where $f_{\theta}$ is the shared feature extractor and $h_{\phi}$ is the domain-aware prediction sub-network.  
At each iteration, we draw a batch
with $N^{\mathcal{B}}$ labeled instances $(X_b,Y_b)=\{(\mathbf{x}_i,\mathbf{y}_i)\}_{i=1}^{N^\mathcal{B}}$ 
from the source domain and generate $K$ surrogate domains $\{X_b^{(k)}\}_{k=1}^K$ via semantically coherent transformation pipelines.

\paragraph{Inner objective.}
Given fixed prompt parameters $\omega$, we update the model parameters $\Theta$ by minimizing a loss that combines standard cross-entropy with an adversarial regularizer that promotes local smoothness under input perturbations. 
The inner supervised loss on a training batch $(X_b,Y_b)=\{(\mathbf{x}_i,\mathbf{y}_i)\}_{i=1}^{N^\mathcal{B}}$ 
 is defined as:
\begin{equation}
\label{eq:inner-cl}
	\mathcal{L}_{\text{in}}^{\text{cl}}(\Theta) = \frac{1}{N^\mathcal{B}} \sum_{i=1}^{N^\mathcal{B}} \ell_{\text{CE}} \big( h_{\phi}(f_{\theta}(\mathbf{x}_i), \mathcal{M}_{\omega}(f_{\bar{\theta}}(X_b))), \mathbf{y}_i \big),
\end{equation}
where $\ell_{\text{CE}}$ is the cross-entropy loss and $\mathcal{M}_{\omega}(X_b)$ denotes 
the domain prompts derived from the batch. To improve robustness, we further include an adversarial KL term that encourages consistency under worst-case perturbations:
\begin{align}
\label{eq:inner-adv}
\mathcal{L}_{\text{in}}^{\text{adv}}(\Theta) \!= \!\frac{1}{N^\mathcal{B}}\!\! \sum_{i=1}^{N^\mathcal{B}} 
{\tiny\! \max_{\|\epsilon\|_2 \le \rho}} \!
\text{KL} 
\Big(\!\!\! 
	{\small
\begin{array}{l}
	h_{\phi}(f_{\theta}(\mathbf{x}_i), \mathcal{M}_{\omega}(f_{\bar{\theta}}(X_b))), \\
	h_{\phi}(f_{\theta}(\mathbf{x}_i \!\!+\! \epsilon), \mathcal{M}_{\omega}(f_{\bar{\theta}}(X_b)))
\end{array}
	}
\!\!\!	\Big),
\end{align}
where $\rho$ controls the strength of the perturbation $\epsilon$ added to input $\mathbf{x}_i$. The total inner objective is:
\begin{equation}
\label{eq:inner-loss}
    \mathcal{L}_{\text{in}}(\Theta,\omega) = \mathcal{L}_{\text{in}}^{\text{cl}}+ \lambda \mathcal{L}_{\text{in}}^{\text{adv}},
\end{equation}
where $\lambda$ balances classification accuracy and adversarial robustness. 
Minimizing the inner loss in Eq.~\eqref{eq:inner-loss} over $\Theta$ 
yields an updated model $\Theta^*(\omega)$ 
given the current prompt encoder.

\paragraph{Outer objective.}
Given model parameters $\Theta^*(\omega)$, 
the outer objective updates $\omega$ by optimizing generalization performance 
across all the surrogate domains. Specifically, for each surrogate domain $k$, 
we sample a batch $(X_b^{(k)}, Y_b)=\{(\mathbf{x}_i^{(k)}, \mathbf{y}_i)\}_{i=1}^{N^\mathcal{B}}$, 
and compute the outer objective as the cross-entropy loss on all the $K$ surrogate domains: 
\begin{align}
\label{eq:outer-loss}
	\mathcal{L}_{\text{out}}(\omega, \Theta^*(\omega)) &= 
	\frac{1}{K \!\cdot\! N^\mathcal{B}}\! \sum_{k=1}^{K} \sum_{i=1}^{N^\mathcal{B}} 
	\ell_{\text{CE}} \left( 
	\hat{\mathbf{y}}_i^{(k)}, 
	\mathbf{y}_i \right),
\\
	\mbox{where}\quad &\hat{\mathbf{y}}_i^{(k)} = 
	 h_{\phi^*}(f_{\theta^*}(\mathbf{x}_i^{(k)}), 
	\mathcal{M}_{\omega}(f_{\bar{\theta^*}}(X_b^{(k)}))).
\nonumber	
\end{align}
Here, $\Theta^*(\omega) = \{\theta^*, \phi^*\}$ denote the inner-optimal parameters
that can be treated as functions of the outer parameters $\omega$.
By minimizing this loss, the prompt encoder learns to 
automatically adapt the prediction model 
to enhance its generalization performance.

\paragraph{Bi-level formulation.}
Combining the inner and outer objectives, BiSDG is formulated
as the following bi-level optimization problem:
\begin{align}
\label{eq:bi-level}
    \min_{\omega} \quad & \mathcal{L}_{\text{out}}(\omega, \Theta^*(\omega)) \\
    \text{s.t.} \quad & \Theta^*(\omega) = \arg\min_{\Theta} \mathcal{L}_{\text{in}}(\Theta,\omega).
    \nonumber
\end{align}
This training scheme enables the model to learn generalizable representations, 
while encouraging the prompt encoder to expose informative cross-domain variability to enhance adaptation. 
The optimization procedure will be illustrated in the following section and 
the bi-level training algorithm in presented in Algorithm~\ref{alg:algorithm}.

\begin{algorithm}[t]
  \caption{Training Algorithm for BiSDG}
  \begin{algorithmic}
    \STATE \textbf{Input}: training dataset $\mathcal{D}$; 
	  initialized parameters $\Theta =(\theta,\phi)$ and $\omega$;
	  $K$ transformation pipelines $\{\mathcal{T}_k\}_{k=1}^K$; 
	  and hyperparameters.\\
	  \STATE \textbf{Output}: learned model parameters  $\Theta^*$ and $\omega^*$ \\[1ex]
    \STATE Generate $K$ surrogate domains $\{\mathcal{D}^{(k)}\}_{k=1}^K$ via Eq.~\eqref{eq:surrogate}.
    \STATE Set $t=1, \Theta^1 = \Theta$.
    \FOR{ iter $=1$ to maxiters}
   
    \FOR{ minibatch $X_b \in \mathcal{D}$}
    	\STATE Sample $K$ surrogate variants $\{X_b^{(k)}\}_{k=1}^K$ of $X_b$ from surrogate domains $\{\mathcal{D}^{(k)}\}_{k=1}^K$. 
    	\STATE Compute inner loss $\mathcal{L}_{\text{in}}$ 
	  on $X_b$  via Eqs.~\eqref{eq:inner-cl}, \eqref{eq:inner-adv}, \eqref{eq:inner-loss} 
	\STATE Calculate $\Theta^{t+1} $ using Eq.~\eqref{eq:single_update}
	\STATE Compute $\mathcal{L}_{\text{out}}(\omega, \Theta^{t+1})$ on surrogate batches $\{X_b^{(k)}\}_{k=1}^K$ via Eq.~\eqref{eq:outer-loss}
	\STATE Calculate $\delta(\Theta^{t+1})$ using Eq.~\eqref{eq:deltatheta}
	\STATE Calculate gradient $\nabla_{\omega}\mathcal{L}_\text{out} $ 
	  on $X_b$ and  $\{X_b^{(k)}\}_{k=1}^K$ via Eq.~\eqref{eq:approximation}
        \STATE Update $\omega \leftarrow \omega- \alpha_\omega\nabla_{\omega}\mathcal{L}_\text{out}$
        \STATE Update $\Theta^{t+1} \leftarrow \Theta^t-\alpha_\Theta \nabla_{\Theta}\mathcal{L}_\text{in}(\Theta^t,\omega)$ 
        \STATE $t\leftarrow t+1$
     \ENDFOR
   
    \ENDFOR
  \STATE $\Theta^* = \Theta^{t}, \quad \omega^*=\omega$
  \end{algorithmic}
\label{alg:algorithm}
\end{algorithm}
\subsection{Optimization Procedure}
The bi-level optimization problem in Eq.~\eqref{eq:bi-level}
can be solved by minimizing the outer objective with respect to
the outer parameters $\omega$, 
while the inner-level parameters $\Theta^*$ are implicit functions of $\omega$ 
determined by the inner minimization. 

In each iteration of the minimization, we compute the gradient of 
the outer loss with respect to $\omega$ and account for the bi-level structure
using the chain rule as follows:
\begin{equation}
\label{eq:outer_derivative}
\nabla_\omega \mathcal{L}_{\text{out}} = \nabla_{\Theta^*} \mathcal{L}_{\text{out}} \cdot \nabla_\omega \Theta^* + \nabla_\omega \mathcal{L}_{\text{out}}(\omega,\bar{\Theta}^*),
\end{equation}
where $\bar{\Theta}^*$ denotes the stop gradient version of $\Theta^*$. 
Here, the first term represents the indirect gradient path through $\Theta^*$
using the chain rule, 
while the second term captures the direct dependence of the outer loss on $\omega$.

For simplicity, 
at the $t$-th iteration, the inner-optimal parameters $\Theta^*$  
can be approximated using a single-step gradient descent update over the inner loss:
\begin{equation}
\begin{aligned}
\label{eq:single_update}
\Theta^* &= \Theta^{t+1} = \Theta^t - \alpha_\Theta \nabla_\Theta \mathcal{L}_{\text{in}}(\Theta^t,\omega),
\end{aligned}
\end{equation}
where $\alpha_\Theta$ is the learning rate for the respective update. 
Moreover, we denote the outer gradient over $\Theta$ as $\delta(\Theta)$,
such that
\begin{equation}
\label{eq:deltatheta}
\delta(\Theta^{t+1}) = \nabla_{\Theta} \mathcal{L}_{\text{out}}(\omega,\Theta^{t+1}).
\end{equation}
The full gradient computation can then be expressed through the following propositions. 
\begin{proposition}
Using the chain rule, 
at the $t$-th iteration, the total gradient of $\mathcal{L}_{\text{out}}$ with respect to $\omega$ is given by:
\begin{equation}
\begin{aligned}
\label{eq:outer-derivitive}
\nabla_\omega \mathcal{L}_{\text{out}} =\;& 
- \alpha_\Theta \cdot \delta(\Theta^{t+1}) \cdot \nabla_\omega \nabla_\Theta \mathcal{L}_{\text{in}}(\Theta^t,\omega) \\
&+ \nabla_\omega \mathcal{L}_{\text{out}}(\omega,\bar{\Theta}^{t+1})
\end{aligned}
\end{equation}
\end{proposition}
\begin{proof}
From Eq.~\eqref{eq:single_update}, the gradient of $\Theta^*$ with respect to $\omega$ can
be computed as:
\begin{equation}
\begin{aligned}
\nabla_\omega \Theta^* = \nabla_\omega \Theta^{t+1} &= \nabla_\omega \left(\Theta^t - \alpha_\Theta \nabla_\Theta \mathcal{L}_{\text{in}}(\Theta^t, \omega) \right) \\
&= -\alpha_\Theta \cdot \nabla_\omega \nabla_\Theta \mathcal{L}_{\text{in}}(\Theta^t,  \omega).
\end{aligned}
\end{equation}
Substituting this expression into Eq.~\eqref{eq:outer_derivative}, we obtain Eq.\eqref{eq:outer-derivitive}.
\end{proof}

Considering the difficulty of computing the second-order derivative 
$\nabla_\omega \nabla_\Theta \mathcal{L}_{\text{in}}(\Theta^t, \omega)$, 
we propose using a finite difference approximation to avoid the expensive computation 
while retaining the influence of second-order terms.
This is summarized in the following proposition. 
\begin{proposition}
Let $\epsilon_\Theta$ be a very small positive constant. The total gradient $\nabla_\omega \mathcal{L}_{\text{out}}$ 
can be approximated as follows without directly computing second-order derivatives:
\begin{equation}
\begin{aligned}
\label{eq:approximation}
\nabla_\omega \mathcal{L}_{\text{out}} \approx
&\;\nabla_\omega \mathcal{L}_{\text{out}}(\omega,\bar{\Theta}^*)- \\
	&\frac{\alpha_\Theta}{2\epsilon_\Theta}\left(\nabla_\omega\mathcal{L}_{\text{in}}(\Theta^+, \omega) - \nabla_\omega\mathcal{L}_{\text{in}}(\Theta^-, \omega)\right),
\end{aligned}
\end{equation}
where the perturbed variables $\Theta^+$ and $\Theta^-$ are defined as:
\begin{equation}
\begin{aligned}
\Theta^+ &= \Theta^t + \epsilon_\Theta \cdot \delta(\Theta^{t+1}), \\
\Theta^- &= \Theta^t - \epsilon_\Theta \cdot \delta(\Theta^{t+1}).
\end{aligned}
\end{equation}
\end{proposition}
\begin{proof}
We approximate the second-order term $\nabla_\omega \nabla_\theta \mathcal{L}_{\text{in}}$ using symmetric finite differences ~\cite{bottou2012stochastic}. Specifically, using the small perturbation $\epsilon_\Theta \cdot \delta(\Theta^{t+1})$, we write:
\begin{equation}
\nabla_\Theta \mathcal{L}_{\text{in}}(\Theta^t, \omega)
\approx \frac{\mathcal{L}_{\text{in}}(\Theta^+,  \omega) - \mathcal{L}_{\text{in}}(\Theta^-, \omega)}{2\epsilon_\Theta \cdot \delta(\Theta^{t+1})}
\end{equation}
Therefore, the chain rule term in the total gradient becomes:
\begin{align}
&\nabla_{\Theta^*} \mathcal{L}_{\text{out}} \cdot \nabla_\omega \Theta^*
	\nonumber\\
	=& -\alpha_\Theta\cdot \delta(\Theta^{t+1}) \cdot \nabla_\omega \nabla_\Theta\mathcal{L}_{\text{in}}(\Theta^t, ,\omega) 
	\nonumber\\
	\approx & -\alpha_\Theta\cdot \delta(\Theta^{t+1}) \cdot \frac{\nabla_\omega \mathcal{L}_{\text{in}}(\Theta^+; \omega) - \nabla_\omega \mathcal{L}_{\text{in}}(\Theta^-, \omega)}{2\epsilon_\Theta\cdot \delta(\Theta^{t+1})} 
	\nonumber\\
	= & -\frac{\alpha_\Theta}{2\epsilon_\Theta} \left(\nabla_\omega \mathcal{L}_{\text{in}}(\Theta^+, \omega) - \nabla_\omega \mathcal{L}_{\text{in}}(\Theta^-,  \omega)\right)
\end{align}
Substituting this approximation into Eq.~\eqref{eq:outer_derivative}, 
we obtain the approximated total gradient in Eq.\eqref{eq:approximation}.
\end{proof}

\begin{table*}[t]
\centering
\small
	\caption{Classification accuracy and standard deviation results (\%) on the PACS dataset. Best results are in bold font.}
\label{tab:pacs-experiment}

\begin{tabular}{l|ccccccccc|c}
\hline
Target & MixUp & CutOut & ADA & ME-ADA & AugMix & RandAug & ACVC& L2D&AdvST   & BiSDG (Ours) \\ \hline
Art & 52.8 & 59.8 & 58.0 & 60.7 & 63.9 & 67.8 & 67.8 & 67.6 & ${69.2_{(1.4)}}$&$\mathbf{70.5}_{(1.4)}$  \\
Cartoon & 17.0 & 21.6 & 25.3 & 28.5 & 27.7 & 28.9 & 30.3 & 42.6 & $55.3_{(2.0)}$&$\mathbf{55.9}_{(1.7)}$  \\
Sketch & 23.2 & 28.8 & 30.1 & 29.6 & 30.9 & 37.0 & 46.4 & 47.1 &$67.7_{(1.5)}$ & $\mathbf{69.5}_{(1.9)}$ \\
Avg. & 31.0 & 36.7 & 37.8 & 39.6 & 40.8 & 44.6 & 48.2 & 52.5 & 64.1 &$\mathbf{65.3}$ \\ \hline
\end{tabular}
\vskip -0.05 in
\end{table*}
\begin{table*}[t]
\centering
\small
	\caption{Classification accuracy and standard deviation results (\%) 
	on the four target domains (SVHN, MNIST-M, SYN, and USPS) of the Digits benchmark, with MNIST as the source domain. Best results are in bold font.}
\label{tab:digits}

\setlength{\tabcolsep}{10pt}
\begin{tabular}{l|cccc|c}
\hline
Method & SVHN & MNIST-M & SYN & USPS & Avg. \\ \hline
 ERM \cite{koltchinskii2011oracle} & 27.8 & 52.7 & 39.7 & 76.9 & 49.3 \\
 CCSA \cite{motiian2017unified}& 25.9 & 49.3 & 37.3 & 83.7 & 49.1\\
JiGen \cite{carlucci2019domain} & 33.8 & 57.8 & 43.8 & 77.2 & 53.1\\
ADA \cite{volpi2018generalizing}& 35.5 & 60.4 & 45.3 & 77.3 & 54.6\\
 ME-ADA \cite{Long2020Maximum} & 42.6 & 63.3 & 50.4 & 81.0 & 59.3\\ 
M-ADA \cite{qiao2020learning} & 42.6 & 67.9 & 49.0& 78.5 & 59.5\\
AutoAug \cite{cubuk2019autoaugment} & 45.2 & 60.5 & 64.5 & 80.6 & 62.7 \\
RandAug \cite{cubuk2020randaugment} & 54.8 & 74.0 & 59.6 & 77.3 & 66.4 \\
  PDEN \cite{li2021progressive} & 62.2 & 82.2 & 69.4 & 85.3 & 74.8\\
    MCL\cite{chen2023meta} & 69.9 & 78.3 & 78.4 & 88.5 & 78.8 \\
    RSDA \cite{volpi2019addressing} & ${47.7}_{(4.8)}$ & ${81.5}_{(1.6)}$ & ${62.0}_{(1.2)}$ & ${83.1}_{(1.2)}$ & 68.5\\
        AdvST \cite{zheng2024advst} & ${67.5}_{(0.7)}$ & ${79.8}_{(0.7)}$ & ${78.1}_{(0.9)}$ & ${94.8}_{(0.4)}$ & 80.1\\
   \hline
BiSDG (Ours)   & $\mathbf{70.1}_{(0.6)}$ & $\mathbf{85.2}_{(0.3)}$ & $\mathbf{79.8}_{(0.5)}$ & $\mathbf{96.0}_{(0.2)}$ & $\mathbf{82.7}$\\

\hline

\end{tabular}
\vskip -0.1 in
\end{table*}

\section{Experiments}
\subsection{Experimental Setup}
\label{section:experimental-setup}
\paragraph{Datasets.} 
We evaluate our approach on three widely-used domain generalization benchmarks: {Digits}, {PACS}, and {DomainNet}.
The \textbf{Digits} benchmark consists of five digit classification datasets, MNIST~\cite{lecun1998gradient}, MNIST-M~\cite{ganin2015unsupervised}, SVHN~\cite{netzer2011reading}, SYN~\cite{ganin2015unsupervised}, and USPS~\cite{denker1989neural}, all sharing the same 10 digit classes (0--9). Following standard practice, we use MNIST as the source domain and evaluate generalization on the remaining four datasets, each exhibiting distinct visual styles and acquisition conditions.
The \textbf{PACS} dataset~\cite{li2017deeper} contains images from four distinct domains, Art, Cartoon, Photo, and Sketch, featuring seven shared object categories. We designate the Photo domain as the source and test generalization to the remaining three domains, which differ substantially in terms of texture, abstraction, and visual style.
The \textbf{DomainNet} dataset~\cite{peng2019moment} is the most challenging benchmark in our study, comprising six highly diverse domains: Real, Infograph, Clipart, Painting, Quickdraw, and Sketch, spanning 345 object categories. We use the Real domain for training and evaluate performance on the remaining five, which exhibit large domain shifts and high intra-class variability.

\begin{table*}[t]
\centering
\caption[DomainNet result]{Classification accuracy and standard deviation results (\%) on the DomainNet dataset.}
\label{tab:domainnet}
{\small 
\begin{tabular}{l|ccccccccc|c}
\hline
Target    & MixUp & CutOut & CutMix & ADA  & ME-ADA & RandAug & AugMix & ACVC       &AdvST   & BiSDG (Ours)                                        \\ \hline
Painting  & 38.6  & 38.3   & 38.3   & 38.2 & 39.0   & 41.3    & 40.8   & 41.3    & $42.3_{(0.1)}$      & $\mathbf{43.1}_{(0.3)}$ \\
Infograph & 13.9  & 13.7   & 13.5   & 13.8 & 14.0   & 13.6    & 13.9   & 12.9          &  ${14.8}_{(0.1)}$&$\mathbf{15.9}_{(0.3)}$  \\
Clipart   & 38.0  & 38.4   & 38.7   & 40.2 & 41.0   & 41.1    & 41.7    & ${42.8}$  &$41.5_{(0.4)}$&$\mathbf{43.5}_{(0.4)} $ \\
Sketch    & 26.0  & 26.2   & 26.9   & 24.8 & 25.3   & 30.4    & 29.8   & 30.9    & $30.8_{(0.3)}$      & $\mathbf{31.9}_{(0.4)}$ \\
Quickdraw & 3.7   & 3.7    & 3.6    & 4.3  & 4.3    & 5.3     & 6.3    & 6.6   & $5.9_{(0.2)}$&$\mathbf{7.1}_{(0.2)}$    \\
Avg.      & 24.0  & 24.1   & 24.2   & 24.3 & 24.7   & 26.3    & 26.5   & 26.9          & $27.1_{(0.2)}$&$\mathbf{28.3}$ \\ \hline
\end{tabular}}

\end{table*}
%
\begin{table*}[t]
\centering
\caption{Ablation study 
results (\%) on the four target domains 
	of the Digits benchmark, with MNIST as the source domain.}
\label{tab:ablation}
{\small
\setlength{\tabcolsep}{12pt}
\begin{tabular}{l|cccc|c}
\hline
Method & SVHN & MNIST-M & SYN & USPS & Avg. \\ \hline
       BiSDG (Ours)   & $\mathbf{70.1}_{(0.6)}$ & $\mathbf{85.2}_{(0.3)}$ & $\mathbf{79.8}_{(0.5)}$ & $\mathbf{96.0}_{(0.2)}$ & $\mathbf{82.7}$\\\hline
        $- \text{w/o } \mathcal{L}_\text{in}^{\text{adv}}$ &${68.1}_{(1.4)}$ & ${80.1}_{(0.6)}$  & ${75.3}_{(0.3)}$  & ${93.3}_{(0.5)}$  & ${79.2}$ \\
             $- \text{w/o standardization } $ &${69.8}_{(0.9)}$ & ${83.5}_{(0.7)}$  & ${77.0}_{(0.4)}$  & ${94.7}_{(0.2)}$  & ${81.2}$ \\
        \hline
        
\end{tabular}}

\vskip -0.1 in
\end{table*}

\paragraph{Implementation Details.}

We follow established protocols for architecture and training across all datasets, consistent with prior work~\cite{wang2021learning}. For the {Digits} benchmark, we employ LeNet as the backbone and train on the first 10{,}000 images from MNIST, with all samples resized to $32 \times 32$ and converted to RGB. The model is trained for 50 epochs using a batch size of 32 and 
initial learning rates of $\alpha_\Theta=10^{-4}$ and $\alpha_\omega=10^{-5}$, 
which are reduced by a factor of 10 after 25 epochs. 
For {PACS}, we fine-tune a ResNet-18 model pre-trained on ImageNet, using images resized to $224 \times 224$. 
Training is conducted for 50 epochs with a batch size of 32 and learning rates of $\alpha_\Theta=10^{-3}$ and $\alpha_\omega=10^{-4}$, 
decayed using a cosine annealing schedule. 
For the {DomainNet} benchmark, we again use ResNet-18 and train for 200 epochs with a batch size of 128, also using cosine learning rate scheduling.
We use the following hyperparameters for BiSDG: adversarial loss weight $\lambda = 0.5$, number of surrogate domains $K = 5$, number of augmentations per surrogate domain $m = 3$, gradient approximation scale $\epsilon_\Theta = 0.01$, and perturbation strength $\rho=1$.
To simulate diverse domain shifts during training, we construct five surrogate domains by composing distinct triplets of augmentations from a shared transformation pool. The Color-Shifted Domain applies HSV shift, contrast adjustment, and solarization to simulate lighting and sensor variation. The Geometric Distortion Domain introduces rotation, translation, and shear to capture spatial deformations. The Photometric Degradation Domain uses inversion, posterization, and histogram equalization to mimic degraded imaging conditions. The Texture Alteration Domain perturbs fine textures via sharpening, cutout, and contrast changes. Lastly, the Scale and Shape Variation Domain applies scaling, rotation, and cutout to model geometric variation and partial occlusion. 
These curated domains introduce controlled but substantial variability, promoting generalization to unseen distributions.
All experiments are repeated five times with different random seeds, and we report the mean accuracy and standard deviation to ensure statistical robustness.

\subsection{Comparison Results}
We evaluate our method against a broad spectrum of baselines, 
including
MixUp~\cite{zhang2018mixup}, CutOut~\cite{devries2017improved}, CutMix~\cite{yun2019cutmix}, AutoAugment~\cite{cubuk2019autoaugment}, RandAugment~\cite{cubuk2020randaugment}, and AugMix~\cite{hendrycks2019augmix}, ACVC~\cite{cugu2022attention},  ERM~\cite{koltchinskii2011oracle}, CCSA~\cite{motiian2017unified}, JiGen~\cite{carlucci2019domain}, ADA~\cite{volpi2018generalizing}, ME-ADA~\cite{Long2020Maximum}, RSDA~\cite{volpi2019addressing}, L2D~\cite{wang2021learning}, PDEN~\cite{li2021progressive}, MCL \cite{chen2023meta}, and AdvST~\cite{zheng2024advst}.

Table~\ref{tab:pacs-experiment} reports the test classification accuracy 
on the three target domains of the PACS dataset using ResNet-18 as the backbone. 
This benchmark features significant domain shifts between the Photo source domain and the target domains of Art, Cartoon, and Sketch. BiSDG achieves the highest accuracy on all three target domains and the best overall average accuracy of 65.3\%, outperforming the strongest baseline, AdvST (64.1\%), by 1.2\%. Notably, BiSDG achieves an accuracy of 55.9\% on the most challenging domain, Cartoon, compared to 55.3\% by AdvST. These results highlight the effectiveness of our method in learning semantically meaningful features that generalize across domains with diverse visuals.

In Table~\ref{tab:digits}, we report generalization results on a digit recognition benchmark, Digits, using MNIST as the source domain and SVHN, MNIST-M, SYN, and USPS as target domains. BiSDG significantly outperforms all prior methods, achieving an average accuracy of 82.7\%, surpassing the second-best method (AdvST at 80.1\%) by a substantial margin. BiSDG obtains the best results on all four target domains, including 70.1\% on SVHN, 85.2\% on MNIST-M, 79.8\% on SYN, and 96.0\% on USPS.

Table~\ref{tab:domainnet} presents results on the challenging DomainNet benchmark, which involves large-scale classification across multiple domains with significant visual diversity. BiSDG achieves the best performance on all five target domains, with an average accuracy of 28.3\%, outperforming AdvST (27.1\%) by 1.2\%. The improvements are especially notable on Infograph (15.9\% vs. 14.8\%) and Quickdraw (7.1\% vs. 5.9\%). These results confirm that BiSDG effectively captures core features that generalize across unseen domains despite substantial distribution shifts.

\subsection{Ablation Studies}
To evaluate the effectiveness of key components in our framework, we design two ablation variants of BiSDG: (1) $- \text{w/o } \mathcal{L}_\text{in}^{\text{adv}}$, which removes the adversarial regularizer from the inner objective to assess its contribution to robust feature learning; 
and (2) $- \text{w/o standardization}$, which disables the feature standardization step applied before domain-aware modulation, in order to examine the role of normalization in stabilizing prompt-guided adaptation.
As shown in Table~\ref{tab:ablation}, both components play a crucial role in the final performance. Removing the adversarial loss $\mathcal{L}_\text{in}^{\text{adv}}$ results in a significant drop in average accuracy from 82.7\% to 79.2\%, confirming that enforcing consistency under perturbations is essential for improving robustness to domain shifts. Similarly, omitting the standardization step reduces performance to 81.2\%, with particularly noticeable declines on SVHN and SYN, two domains that exhibit large visual disparities from MNIST. 
These findings demonstrate that both adversarial consistency and standardized feature modulation are critical to BiSDG’s strong generalization performance.
 \begin{figure}[t]
\centering
\begin{subfigure}{0.20\textwidth}
\centering
\includegraphics[width = \textwidth, height=1.in]{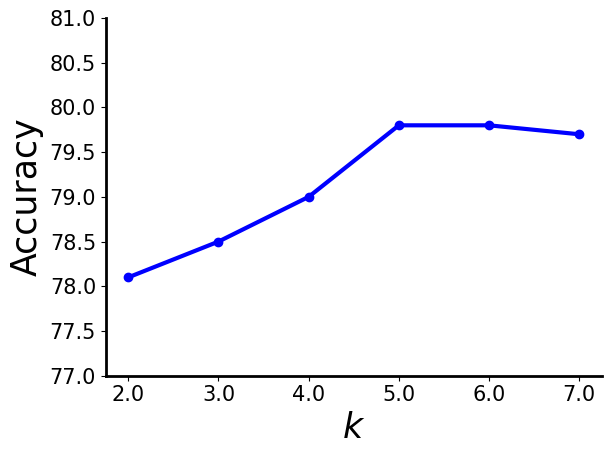}
\caption{$K$}
\end{subfigure}
\begin{subfigure}{0.20\textwidth}
\centering
\includegraphics[width = \textwidth, height=1.in]{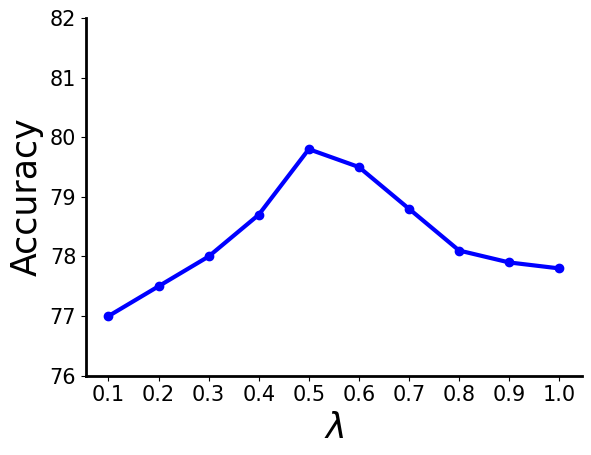}
\caption{$\lambda$}
\end{subfigure}
\begin{subfigure}{0.20\textwidth}
\centering
\includegraphics[width = \textwidth, height=1.in]{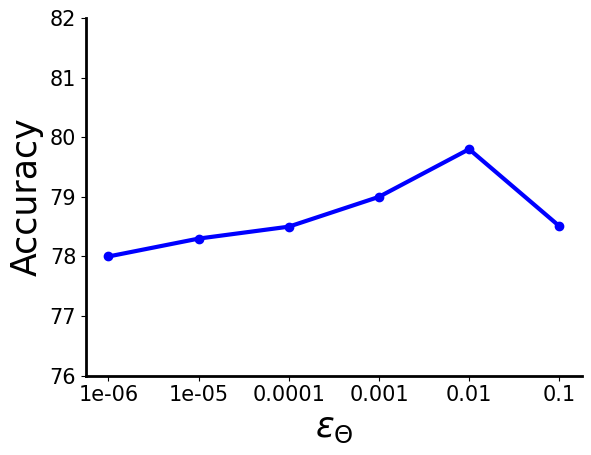}

\caption{$\epsilon_\Theta$}
\end{subfigure}
\begin{subfigure}{0.20\textwidth}
\centering
\includegraphics[width = \textwidth, height=1.in]{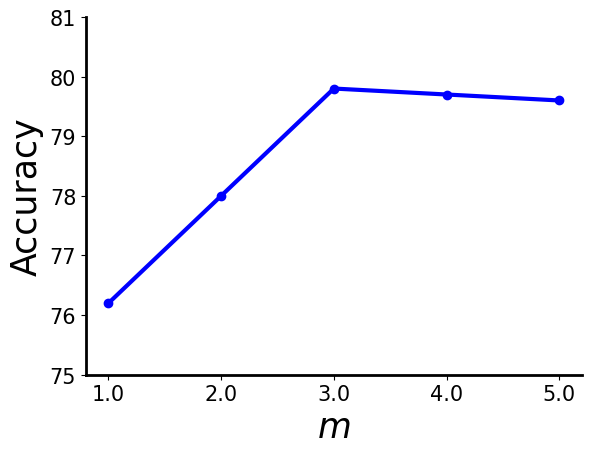}

\caption{$m$}
\end{subfigure}
\caption{Sensitivity analysis for four hyper-parameters  $K$, $\lambda$, $\epsilon_\Theta$ and $m$ on 
	 test domain SYN of the Digits benchmark.}

\label{fig:hyper_sen}
 \vskip -.1in
\end{figure}
\subsection{Hyper-parameter Sensitivity Analysis}
We perform a comprehensive sensitivity analysis to assess the robustness of BiSDG with respect to key hyperparameters. Specifically, we investigate the impact of the number of surrogate domains $K$,  the trade-off weight $\lambda$ for the adversarial loss in the inner objective, 
the perturbation magnitude $\epsilon_\Theta$ used in the finite difference gradient approximation, and the number of augmentations $m$ used to construct each surrogate domain. 
All experiments are conducted on the Digits benchmark, using classification accuracy on the SYN domain as the evaluation metric.

As shown in Figure~\ref{fig:hyper_sen}, BiSDG exhibits relatively stable performance across a wide range of hyperparameter settings. Increasing the number of surrogate domains $K$ (Figure~\ref{fig:hyper_sen}a) and the number of augmentations $m$ (Figure~\ref{fig:hyper_sen}d) initially improves performance, reaching optimal values at $K=5$ and $m=3$, respectively. Beyond these points, performance shows a slight decline as the hyperparameter values increase, but the overall drop is marginal, indicating that BiSDG maintains strong generalization capability without requiring precise tuning of these parameters. 
The trade-off weight $\lambda$ and perturbation magnitude $\epsilon_\Theta$ exhibit a similar trend (Figures~\ref{fig:hyper_sen}b and \ref{fig:hyper_sen}c), though performance degrades more rapidly after reaching their optimal values at $\lambda=0.5$ and $\epsilon_\Theta=0.01$. This suggests that these two hyperparameters are more sensitive and play a more critical role in determining model performance. 

\section{Conclusion}
We have introduced BiSDG, a bi-level optimization framework designed for the challenging setting of Single Domain Generalization (SDG). By simulating domain shifts through curated surrogate domains and decoupling domain-specific modulation from task learning, BiSDG effectively exposes the model to distributional variability while maintaining semantic consistency. The proposed domain prompt encoder enables lightweight yet expressive domain-aware adaptation.
Our bi-level formulation ensures that the backbone focuses on robust representation learning while the prompt encoder is optimized to induce generalization over domain shifts.
Extensive experiments on multiple benchmarks demonstrate that BiSDG consistently outperforms existing state-of-the-art approaches.

{
 \small
 \bibliographystyle{ieeenat_fullname}
 \bibliography{ref}

@String(CVPR= {IEEE Conf. Comput. Vis. Pattern Recog.})

@String(ICCV= {Int. Conf. Comput. Vis.})

@String(ECCV= {Eur. Conf. Comput. Vis.})

@String(ICLR = {Int. Conf. Learn. Represent.})

@String(AAAI = {AAAI})

@String(CVPRW= {IEEE Conf. Comput. Vis. Pattern Recog. Worksh.})

@String(CVPR  = {CVPR})

@String(ICCV  = {ICCV})

@String(ECCV  = {ECCV})

@String(ICLR  = {ICLR})

@String(CVPRW= {CVPRW})

@inproceedings{heidari25,
  title={Bi-Level Optimization for Semi-Supervised Learning with Pseudo-Labeling},
  author={Marzi Heidari and Yuhong Guo},
  booktitle={AAAI conference on artificial intelligence (AAAI)},
  year={2025}
}

@inproceedings{qiao2020learning,
  title={Learning to learn single domain generalization},
  author={Qiao, Fengchun and Zhao, Long and Peng, Xi},
  booktitle={IEEE/CVF Conference on Computer Vision and Pattern Recognition (CVPR)},
  year={2020}
}

@inproceedings{li2021progressive,
  title={Progressive domain expansion network for single domain generalization},
  author={Li, Lei and Gao, Ke and Cao, Juan and Huang, Ziyao and Weng, Yepeng and Mi, Xiaoyue and Yu, Zhengze and Li, Xiaoya and Xia, Boyang},
  booktitle={IEEE/CVF Conference on Computer Vision and Pattern Recognition (CVPR)},
  year={2021}
}

@inproceedings{zheng2024advst,
  title={{AdvST:} Revisiting data augmentations for single domain generalization},
  author={Zheng, Guangtao and Huai, Mengdi and Zhang, Aidong},
  booktitle={AAAI conference on artificial intelligence (AAAI)},

  year={2024}
}

@inproceedings{motiian2017unified,
  title={Unified deep supervised domain adaptation and generalization},
  author={Motiian, Saeid and Piccirilli, Marco and Adjeroh, Donald A and Doretto, Gianfranco},
  booktitle={International Conference on Computer Vision (ICCV)},
  year={2017}
}

@inproceedings{lee2019set,
  title={Set transformer: A framework for attention-based permutation-invariant neural networks},
  author={Lee, Juho and Lee, Yoonho and Kim, Jungtaek and Kosiorek, Adam and Choi, Seungjin and Teh, Yee Whye},
  booktitle={International conference on machine learning (ICML)} ,
  year={2019},
  organization={PMLR}
}

@inproceedings{perez2018film,
  title={Film: Visual reasoning with a general conditioning layer},
  author={Perez, Ethan and Strub, Florian and De Vries, Harm and Dumoulin, Vincent and Courville, Aaron},
  booktitle={AAAI conference on artificial intelligence (AAAI)},
  year={2018}
}

@inproceedings{pedregosa2016hyperparameter,
  title={Hyperparameter optimization with approximate gradient},
  author={Pedregosa, Fabian},
  booktitle={International conference on machine learning (ICML)},
  year={2016},
}

@article{jia2024meta,
  title={Meta-learning the invariant representation for domain generalization},
  author={Jia, Chen and Zhang, Yue},
  journal={Machine Learning},
  volume = {113},
  pages={1661--1681},
  year={2024},
}

@inproceedings{qin2023bi,
  title={Bi-level meta-learning for few-shot domain generalization},
  author={Qin, Xiaorong and Song, Xinhang and Jiang, Shuqiang},
  booktitle={IEEE/CVF conference on computer vision and pattern recognition (CVPR)},
  year={2023}
}

@inproceedings{liu2019darts,
  title={{DARTS}: Differentiable Architecture Search},
  author={Liu, Hanxiao and Simonyan, Karen and Yang, Yiming},
  booktitle={International Conference on Learning Representations (ICLR)},
    year={2019}
}

@inproceedings{peng2024advancing,
  title={Advancing open-set domain generalization using evidential bi-level hardest domain scheduler},
  author={Peng, Kunyu and Wen, Di and Yang, Kailun and Luo, Ao and Chen, Yufan and Fu, Jia and Sarfraz, M Saquib and Roitberg, Alina and Stiefelhagen, Rainer},
  booktitle={Advances in Neural Information Processing Systems (NeurIPS)} ,
  year={2024}
}

@article{bottou2012stochastic,
  title={Stochastic gradient descent tricks},
  author={Bottou, L{\'e}on},
  journal={Neural Networks: Tricks of the Trade: Second Edition},
  year={2012},
  pages = {421–436},
  publisher={Springer}
}

@inproceedings{du2020learning,
  title={Learning to learn with variational information bottleneck for domain generalization},
  author={Du, Yingjun and Xu, Jun and Xiong, Huan and Qiu, Qiang and Zhen, Xiantong and Snoek, Cees GM and Shao, Ling},
  booktitle={European Conference on Computer Vision (ECCV)},
  year={2020},
}

@inproceedings{volpi2018generalizing,
  title={Generalizing to unseen domains via adversarial data augmentation},
  author={Volpi, Riccardo and Namkoong, Hongseok and Sener, Ozan and Duchi, John C and Murino, Vittorio and Savarese, Silvio},
  booktitle={Advances in Neural Information Processing Systems (NeurIPS)},
  year={2018}
}

@inproceedings{carlucci2019domain,
  title={Domain generalization by solving jigsaw puzzles},
  author={Carlucci, Fabio M and D'Innocente, Antonio and Bucci, Silvia and Caputo, Barbara and Tommasi, Tatiana},
  booktitle={IEEE/CVF Conference on Computer Vision and Pattern Recognition (CVPR)},
  year={2019}
}

@inproceedings{mahajan2021domain,
  title={Domain generalization using causal matching},
  author={Mahajan, Divyat and Tople, Shruti and Sharma, Amit},
  booktitle={International Conference on Machine Learning (ICML)},
  year={2021},
  
}

@inproceedings{cubuk2019autoaugment,
  title={AutoAugment: Learning Augmentation Strategies From Data},
  author={Cubuk, Ekin D and Zoph, Barret and Mane, Dandelion and Vasudevan, Vijay and Le, Quoc V},
  booktitle={IEEE/CVF Conference on Computer Vision and Pattern Recognition (CVPR)},
  year={2019}
}

@inproceedings{cubuk2020randaugment,
  title={RandAugment: Practical Automated Data Augmentation with a Reduced Search Space},
  author={Cubuk, Ekin Dogus and Zoph, Barret and Shlens, Jon and Le, Quoc},
  booktitle={Advances in Neural Information Processing Systems (NeurIPS)},
  year={2020}
}

@article{lecun1998gradient,
  title={Gradient-based learning applied to document recognition},
  author={LeCun, Yann and Bottou, L{\'e}on and Bengio, Yoshua and Haffner, Patrick},
  journal={Proceedings of the IEEE},
  volume = {86, no, 11},
  pages = {2278--2324},
  year={1998},
}

@article{netzer2011reading,
  title={Reading digits in natural images with unsupervised feature learning},
  author={Netzer, Yuval and Wang, Tao and Coates, Adam and Bissacco, Alessandro and Wu, Bo and Ng, Andrew Y},
  journal={NeurIPS workshop on deep learning and unsupervised feature learning},
  year={2011}
}

@inproceedings{yun2019cutmix,
  title={Cutmix: Regularization strategy to train strong classifiers with localizable features},
  author={Yun, Sangdoo and Han, Dongyoon and Oh, Seong Joon and Chun, Sanghyuk and Choe, Junsuk and Yoo, Youngjoon},
  booktitle={International Conference on Computer Vision (ICCV)},
  year={2019}
}

@inproceedings{ganin2015unsupervised,
  title={Unsupervised domain adaptation by backpropagation},
  author={Ganin, Yaroslav and Lempitsky, Victor},
  booktitle={International Conference on Machine Learning (ICML)},
  year={2015},
  
}

@inproceedings{li2017deeper,
  title={Deeper, broader and artier domain generalization},
  author={Li, Da and Yang, Yongxin and Song, Yi-Zhe and Hospedales, Timothy M},
  booktitle={International Conference on Computer Vision (ICCV)},
  year={2017}
}

@book{koltchinskii2011oracle,
  title={Oracle inequalities in empirical risk minimization and sparse recovery problems: {\'E}cole D’{\'E}t{\'e} de Probabilit{\'e}s de Saint-Flour XXXVIII-2008},
  author={Koltchinskii, Vladimir},
  year={2011},
  publisher={Springer}
}

@inproceedings{Long2020Maximum,
 author = {Zhao, Long and Liu, Ting and Peng, Xi and Metaxas, Dimitris},
 title = {Maximum-Entropy Adversarial Data Augmentation for Improved Generalization and Robustness},
 booktitle = {Advances in Neural Information Processing Systems (NeurIPS)},
 year = {2020}
}

@inproceedings{volpi2019addressing,
  title={Addressing model vulnerability to distributional shifts over image transformation sets},
  author={Volpi, Riccardo and Murino, Vittorio},
  booktitle={International Conference on Computer Vision (ICCV)},
  year={2019}
}

@inproceedings{chen2023meta,
  title={Meta-causal learning for single domain generalization},
  author={Chen, Jin and Gao, Zhi and Wu, Xinxiao and Luo, Jiebo},
  booktitle={IEEE/CVF Conference on Computer Vision and Pattern Recognition (CVPR)},
  year={2023}
}

@inproceedings{zhang2018mixup,
  title={mixup: Beyond Empirical Risk Minimization},
  author={Zhang, Hongyi and Cisse, Moustapha and Dauphin, Yann N and Lopez-Paz, David},
  booktitle={International Conference on Learning Representations (ICLR)},
  year={2018}
}

@inproceedings{denker1989neural,
  title={Neural network recognizer for hand-written zip code digits},
  author={Denker, John S and Gardner, WR and Graf, Hans Peter and Henderson, Donnie and Howard, Richard E and Hubbard, W and Jackel, Lawrence D and Baird, Henry S and Guyon, Isabelle},
  booktitle={Advances in Neural Information Processing Systems (NeurIPS)},
  year={1989},
}

@inproceedings{wang2021learning,
  title={Learning to diversify for single domain generalization},
  author={Wang, Zijian and Luo, Yadan and Qiu, Ruihong and Huang, Zi and Baktashmotlagh, Mahsa},
  booktitle={International Conference on Computer Vision (ICCV)},
  year={2021}
}

@inproceedings{cugu2022attention,
  title={Attention consistency on visual corruptions for single-source domain generalization},
  author={Cugu, Ilke and Mancini, Massimiliano and Chen, Yanbei and Akata, Zeynep},
  booktitle={IEEE/CVF Conference on Computer Vision and Pattern Recognition Workshops (CVPRW)},
  year={2022}
}

@inproceedings{hendrycks2019augmix,
  title={AugMix: A simple data processing method to improve robustness and uncertainty},
  author={Hendrycks, Dan and Mu, Norman and Cubuk, Ekin Dogus and Zoph, Barret and Gilmer, Justin and Lakshminarayanan, Balaji},
  booktitle={International Conference on Learning Representations (ICLR)},
  year={2019}
}

@inproceedings{peng2019moment,
  title={Moment matching for multi-source domain adaptation},
  author={Peng, Xingchao and Bai, Qinxun and Xia, Xide and Huang, Zijun and Saenko, Kate and Wang, Bo},
  booktitle={International Conference on Computer Vision (ICCV)},
  year={2019}
}

@inproceedings{li2018learning,
  title={Learning to generalize: Meta-learning for domain generalization},
  author={Li, Da and Yang, Yongxin and Song, Yi-Zhe and Hospedales, Timothy M},
  booktitle={AAAI Conference on Artificial Intelligence (AAAI)},
  year={2018}
}

@article{devries2017improved,
  title={Improved regularization of convolutional neural networks with cutout},
  author={DeVries, Terrance and Taylor, Graham W},
  journal={arXiv preprint arXiv:1708.04552},
  year={2017}
}

@inproceedings{zhang2023adversarial,
  title={Adversarial style augmentation for domain generalization},
  author={Zhang, Yabin and Deng, Bin and Li, Ruihuang and Jia, Kui and Zhang, Lei},
  booktitle={arXiv preprint arXiv:2301.12643},
  year={2023}
}

@inproceedings{xu2023simde,
  title={Simde: A simple domain expansion approach for single-source domain generalization},
  author={Xu, Qinwei and Zhang, Ruipeng and Wu, Yi-Yan and Zhang, Ya and Liu, Ning and Wang, Yanfeng},
  booktitle={IEEE/CVF Conference on Computer Vision and Pattern Recognition (CVPR)},
  year={2023}
}

@inproceedings{li2018deep,
  title={Deep domain generalization via conditional invariant adversarial networks},
  author={Li, Ya and Tian, Xinmei and Gong, Mingming and Liu, Yajing and Liu, Tongliang and Zhang, Kun and Tao, Dacheng},
  booktitle={European Conference on Computer Vision (ECCV)},
  year={2018}
}

@inproceedings{cha2021swad,
  title={Swad: Domain generalization by seeking flat minima},
  author={Cha, Junbum and Chun, Sanghyuk and Lee, Kyungjae and Cho, Han-Cheol and Park, Seunghyun and Lee, Yunsung and Park, Sungrae},
  booktitle={Advances in Neural Information Processing Systems (NeurIPS)},
  year={2021}
}

@inproceedings{zhang2023flatness,
  title={Flatness-aware minimization for domain generalization},
  author={Zhang, Xingxuan and Xu, Renzhe and Yu, Han and Dong, Yancheng and Tian, Pengfei and Cui, Peng},
  booktitle={International Conference on Computer Vision (ICCV)},
  year={2023}
}
}

\end{document}